\documentclass{article}

\usepackage[preprint]{aaai21}

\usepackage[pdftex]{graphicx}
\usepackage[utf8]{inputenc} 
\usepackage[T1]{fontenc}    
\usepackage{hyperref}       

\usepackage{url}            
\usepackage{booktabs}       
\usepackage{amsfonts}       
\usepackage{nicefrac}       
\usepackage{microtype}      
\usepackage{wrapfig}
\usepackage{amsmath}
\usepackage{amsthm}
\usepackage{amssymb}
\usepackage{algorithm}
\usepackage{algorithmic}
\newtheorem{Def}{Definition}[section]
\newtheorem{proposition}{Proposition}[section]

\title{Towards Effective Context for Meta-Reinforcement Learning: an Approach based on Contrastive Learning}
\author{
	\centerline{Haotian Fu\textsuperscript{\rm 1}\thanks{Work done as an intern at Noah's Ark Lab, Huawei}~, Hongyao Tang\textsuperscript{\rm 1}, Jianye Hao\textsuperscript{\rm 1,2}, Chen Chen\textsuperscript{\rm 2}, Xidong Feng\textsuperscript{\rm 3}, Dong Li\textsuperscript{\rm 2}, Wulong Liu\textsuperscript{\rm 2}}
	\\
	\centerline{\textsuperscript{\rm 1}Tianjin University, \textsuperscript{\rm 2}Noah's Ark Lab, Huawei,\textsuperscript{\rm 3}Department of Automation, Tsinghua University} \\
	\{haotianfu, bluecontra\}@tju.edu.cn,\\ \{haojianye, chenchen9, lidong106, liuwulong\}@huawei.com, fengxidongwh@gmail.com
}

\begin{document}
	\renewcommand{\qedsymbol}{}
	\maketitle
	
	\begin{abstract}
		Context, the embedding of previous collected trajectories, is a powerful construct for Meta-Reinforcement Learning (Meta-RL) algorithms. By conditioning on an effective context, Meta-RL policies can easily generalize to new tasks within a few adaptation steps. We argue that improving the quality of context involves answering two questions: 1. How to train a compact and sufficient encoder that can embed the task-specific information contained in prior trajectories? 2. How to collect informative trajectories of which the corresponding context reflects the specification of tasks? To this end, we propose a novel Meta-RL framework called CCM (\textbf{C}ontrastive learning augmented \textbf{C}ontext-based \textbf{M}eta-RL). We first focus on the contrastive nature behind different tasks and leverage it to train a compact and sufficient context encoder. Further, we train a separate exploration policy and theoretically derive a new information-gain-based objective which aims to collect informative trajectories in a few steps. Empirically, we evaluate our approaches on common benchmarks as well as several complex sparse-reward environments. The experimental results show that CCM outperforms state-of-the-art algorithms by addressing previously mentioned problems respectively.  
	\end{abstract}
	
	\section{Introduction}
	Reinforcement Learning (RL) combined with deep neural networks has achieved impressive results on various complex tasks~\citep{DBLP:journals/nature/MnihKSRVBGRFOPB15,DBLP:journals/corr/LillicrapHPHETS15,DBLP:conf/icml/SchulmanLAJM15}. Conventional RL agents need large amount of environmental interactions to train a single policy for one task. However, in real-world problems many tasks share similar internal structures and we expect agents to adapt to such tasks quickly based on prior experiences. Meta-Reinforcement Learning (Meta-RL) proposes to address such problems by learning how to learn~\citep{DBLP:journals/corr/WangKTSLMBKB16}. Given a number of tasks with similar structures, Meta-RL methods aim to capture such common knowledge from previous experience on training tasks and adapt to a new task with only a small amount of interactions.
	
	Based on this idea, many Meta-RL methods try to learn a general model initialization and update the parameters during adaptation~\citep{DBLP:conf/icml/FinnAL17,DBLP:conf/iclr/RothfussLCAA19}. Such methods require on-policy meta-training and are empirically proved to be sample inefficient. To alleviate this problem, a number of methods~\citep{DBLP:conf/icml/RakellyZFLQ19,DBLP:conf/iclr/FakoorCSS20} are proposed to meta-learn a policy that is able to adapt with off-policy data by leveraging context information. Typically, an agent adapts to a new environment by inferring latent context from a small number of interactions with the environment. The latent context is expected to be able to capture the distribution of tasks and efficiently infer new tasks. Context-based Meta-RL methods then train a policy conditioned on the latent context to improve generalization.
	
	As the key component of context-based Meta-RL, the quality of latent context can affect algorithms' performance significantly. However, current algorithms are sub-optimal in two aspects. Firstly, the training strategy for context encoder is flawed. A desirable context is expected to only extract task-specific information from trajectories and throw away other information. However, the latent context learned by existing methods (i.e. recovering value function~\citep{DBLP:conf/icml/RakellyZFLQ19} or dynamics prediction~\citep{DBLP:journals/corr/abs-2005-06800,DBLP:conf/iclr/ZhouPG19}) are quite noisy as it may model irrelevant dependencies and ignore some task-specific information. Instead, we propose to directly analyze and discriminate the underlying structure behind different tasks' trajectories by leveraging contrastive learning. Secondly, prior context-based Meta-RL methods ignore the importance of collecting informative trajectories for generating distinctive context. If the exploration process does not collect transitions that are able to reflect the task's individual property and distinguish it from dissimilar tasks, the latent context would be ineffective. For instance, in many cases, tasks in one distribution only vary in the final goals, which means the transition dynamics remains the same in most places of the state space. Without a good exploration policy it is hard to obtain information that is able to distinguish tasks from each other, which leads to a bad context.
	
	In this paper, we propose a novel off-policy Meta-RL algorithm CCM (\textbf{C}ontrastive learning augmented \textbf{C}ontext-based \textbf{M}eta-RL), aiming to improve the quality of context by tackling the two aforementioned problems. Our first contribution is an unsupervised training framework for context encoder by leveraging contrastive learning. The main insight is that by setting transitions from the same task as positive samples and the ones from different tasks as negative samples, contrastive learning is able to directly distinguish context in the original latent space without modeling irrelevant dependencies. The second contribution is an information-gain-based exploration strategy. With the purpose of collecting trajectories as informative as possible, we theoretically obtain a lower bound estimation of the exploration objective in contrastive learning framework. Then it is employed as an intrinsic reward, based on which a separate exploration agent is trained. The effectiveness of CCM is validated on a variety of continuous control tasks. The experimental results show that CCM outperforms state-of-the-art Meta-RL methods through generating high-quality latent context.
	\section{Preliminaries}
	\subsection{Meta-Reinforcement Learning}
	\label{2.1}
	In meta-reinforcement learning (Meta-RL) scenario, we assume a distribution of tasks $p(\mu)$. Each task $\mu \sim p(\mu)$ shares similar structures and corresponds to a different Markov Decision Process (MDP), $M_{\mu} = \{S, A, T_{\mu}, R_{\mu}\}$, with state space $S$, action space $A$, transition distribution $T_{\mu}$, and reward function $R_{\mu}$. We assume one or both of transition dynamics and reward function vary across tasks. Following prior problem settings in~\citep{DBLP:journals/corr/DuanSCBSA16,DBLP:conf/iclr/FakoorCSS20,DBLP:conf/icml/RakellyZFLQ19}, we define a meta-test trial as $N$ episodes in the same MDP, with an initial exploration of $K$ episodes, followed by $N-K$ execution episodes leveraging the data collected in exploration phase. We further define a transition batch sampled from task $\mu$ as $b = \{\tau_{i}\}_{i=1}^{k}$. A trajectory consists of $T$ transitions $\tau_{1:T}$ is a special case of transition batch when all transitions are consecutive. For context-based Meta-RL, the agent's policy typically depends on all prior transitions $\tau_{1:\Gamma} = \{(s_{1},a_{1}, r_{1},s'_{1})\cdots(s_{\Gamma},a_{\Gamma}, r_{\Gamma},s'_{\Gamma})\}$ collected by exploration policy $\pi_{exp}$. The agent firstly consumes the collected trajectories and outputs a latent context $z$ through context encoder $q(z|\tau_{1:\Gamma})$, then executes policy $\pi_{exe}$ conditioned on the current state and latent context $z$. The goal of the agent is therefore to maximize the expected return, $\mathbb{E}_{\mu\sim p(\mu)}\big[\mathbb{E}_{\tau_{1:\Gamma}\sim\pi_{exp}}\big[V_{\pi_{exe}}(q(z|\tau_{1:\Gamma}), \mu)\big]\big]$.
	
	\subsection{Contrastive Representation Learning}
	\label{contraspre}
	The key component of representation learning is the way to efficiently learn rich representations of given input data. Contrastive learning has recently been widely used to achieve such purpose. The core idea is to learn representation function that maps semantically similar data closer in the embedding space. Given a query $q$ and keys $\{ k_{0}, k_{1}, \cdots \}$, the goal of contrastive representation learning is to ensure $q$ matches with positive key $k_{i}$ more than any other keys in the data set. Empirically, positive keys and query are often obtained by taking two augmented versions of the same image, and negative keys are obtained from other images. Previous work proposes InfoNCE loss~\citep{DBLP:journals/corr/abs-1807-03748} to score positive keys $k_{i}\sim p(k|q)$ higher compared with a set of $K$ distractors $k_{j}\sim p(k)$:
	\begin{equation}
	\begin{aligned}
	L_{NCE} &= -\mathbb{E}\Big[\log\frac{\exp(f(q, k_{i}))}{\sum _{j=1}^{K}\exp(f(q, k_{j}))} \Big],
	\end{aligned}
	\label{1nce}
	\end{equation}
	where function $f$ calculates the similarity score between query and key data, and is usually modeled as bilinear products, i.e. $q^{T}Wk$~\citep{DBLP:journals/corr/abs-1905-09272}. As proposed and proved in~\citep{DBLP:journals/corr/abs-1807-03748}, minimizing the InfoNCE loss is equivalent to maximizing a lower bound of the mutual information between $q$ and $k$:
	\begin{equation}
	I(q;k)\geqslant \log(K) - L_{NCE}
	\label{maxlow1}
	\end{equation}
	The lower bound becomes tighter as $K$ increases.
	
	\section{Contrastive Context Encoder}
	As the key component in context-based Meta-RL framework, how to train a powerful context encoder is non-trivial. One straightforward method is to train the encoder in an end-to-end fashion from RL loss (i.e. recover the state-action value function~\citep{DBLP:conf/icml/RakellyZFLQ19}). However, the update signals from value function is stochastic and weak that may not capture the similarity relations among tasks. Moreover, recovering value function is only able to implicitly capture high-level task-specific features and may ignore low-level detailed transition difference that contains relevant information as well. Another kind of methods resorts to dynamics prediction~\citep{DBLP:journals/corr/abs-2005-06800,DBLP:conf/iclr/ZhouPG19}. The main insight is to capture the task-specific features via distinguishing varying dynamics among tasks. However, entirely depending on low-level reconstructing states or actions is prone to over-fit on the commonly-shared dynamics transitions and model irrelevant dependencies which may hinder the learning process.

	These two existing methods train the context encoder by extracting high-level or low-level task-specific features but either not sufficient because of ignoring useful information or not compact because of modeling irrelevant dependencies. To this end, here we aim to train a compact and sufficient encoder through extracting mid-level task-specific features. We propose to directly extract the underlying task structure behind trajectories by performing contrastive learning on the nature distinctions of trajectories. Through explicitly comparing different tasks' trajectories instead of each transition's dynamics, the encoder is prevented from modeling commonly-shared information while still be able to capture all the relevant task-specific structures by leveraging the contrastive nature behind different tasks.

	Contrastive learning methods learn representations by pulling together semantically similar data points (positive data pairs) while pushing apart dissimilar ones (negative data pairs). We treat the trajectories sampled from same tasks as the positive data pairs and trajectories from different tasks as negative data pairs. Then the contrastive loss is minimized to gather the latent context of trajectories from same tasks closer in embedding space while pushing apart the context from other tasks.

	Concretely, assuming a training task set containing $M$ different tasks from task distribution $p(\mu)$. We first generate trajectories with current policy for each task and store them in replay buffer. At each training step, we first sample a task $\mu_{n}$ from the training task set, and then randomly sample two batches of transitions $b^{q}_{n}$, $b^{k}_{n}$ from task $\mu_{n}$ independently. $b^{q}_{n}$ serves as a query in contrastive learning framework while $b^{k}_{n}$ is the corresponding positive key. We also randomly sample $M-1$ batches of transitions from the other tasks as negative keys. Note that previous work (e.g. CURL~\citeyearpar{DBLP:journals/corr/abs-2004-04136}) relies on data augmentation on images like random crop to generate positive data. In contrast, independent transitions sampled from the same task replay buffer naturally construct the positive samples.
	
	As shown in Figure~\ref{fig1}, after obtaining query and key data, we map them into latent context $z^{q}$ and $z^{k}$ respectively, on which we calculate similarity score and contrastive loss to train the encoder:
	\begin{equation}
	\begin{aligned}
	L_{contrastive} &= -\mathbb{E}\Big[\log\frac{\exp(f(z^{q}, z^{k}_{n}))}{\sum _{j=1}^{M}\exp(f(z^{q}, z^{k}_{j}))} \Big]
	\end{aligned}
	\label{1nce}
	\end{equation}
	Following the settings in~\citep{DBLP:conf/cvpr/He0WXG20,DBLP:journals/corr/abs-2004-04136}, we use momentum averaged version of the query encoder $e^{q}$ for encoding the keys. The Reinforcement Learning part takes in the latent context as an additional observation, then executes policy and updates separately. Note that the Contrastive Context Encoder is a generic framework and can be integrated with any context-based Meta-RL algorithm.
	
	\begin{figure}[t]
		\centering
		\begin{minipage}[t]{0.44\textwidth}
			\centering
			\includegraphics[width=0.99\linewidth]{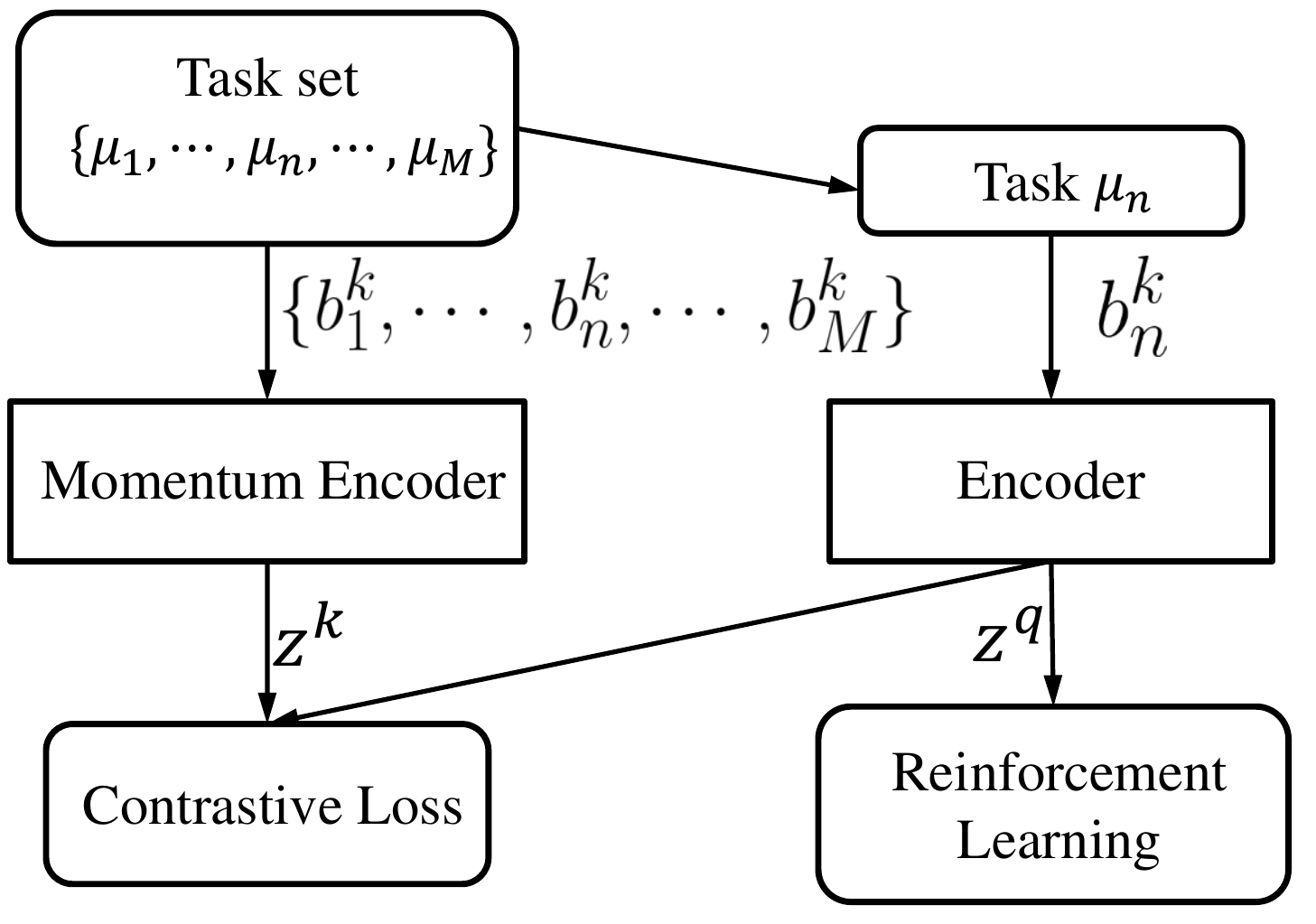}
			\caption{Contrastive Context Encoder}
			\label{fig1}
		\end{minipage}
		\begin{minipage}[t]{0.55\textwidth}
			\centering
			\includegraphics[width=0.99\textwidth]{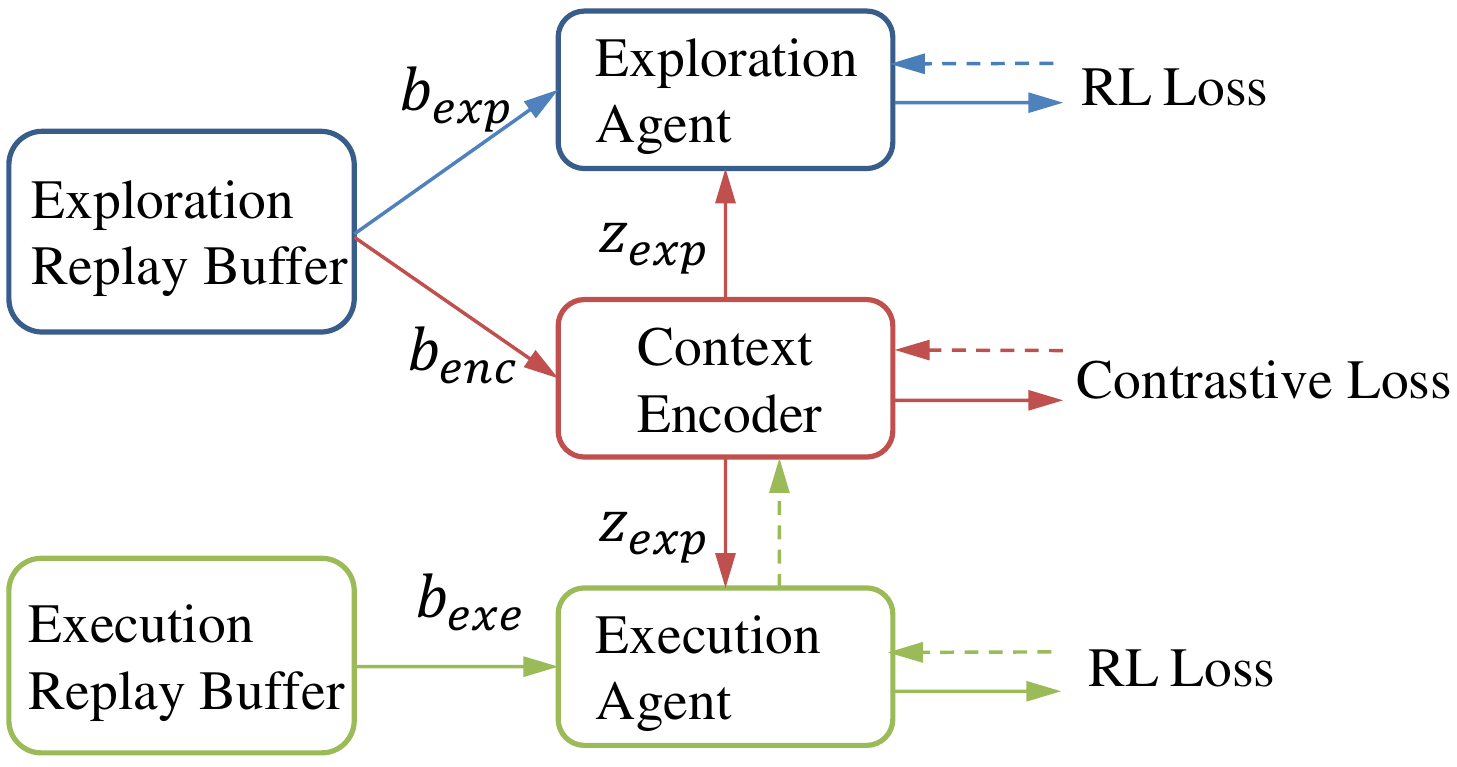}
			\caption{CCM training procedure. Dashed lines denote backward gradients.}
			\label{fig2}
		\end{minipage}
	\end{figure}
	
	\section{Information-gain-based Exploration for Effective Context}
	Even with a compact and sufficient encoder, the context is still ineffective if the exploration process does not collect enough transitions that is able to reflect the new task's specific property and distinguish it from dissimilar tasks. Some previous approaches~\citep{DBLP:conf/icml/RakellyZFLQ19,DBLP:conf/iclr/RothfussLCAA19} utilize Thompson-sampling~\citep{Schmidhuber1987EvolutionaryPI} to explore, in which an agent needs to explore in the initial few episodes and execute optimally in the subsequent episode. However, these two processes is actually conducted by one single policy and is trained in an end-to-end fashion by maximizing expected return. This means the exploration policy is limited to the learned execution policy. When adapting to a new task, the agent tends to only explore experiences which are useful for solving previously trained tasks, making the adaptation process less effective.
	
	Instead, we decouple the exploration and execution policy and define an exploration agent aiming to collect trajectories as informative as possible. We achieve this goal by encouraging the exploration policy to maximize the \textit{information gain} from collecting new transition $\tau_{i}$ at time step $i$ of task $\mu_{n}$:
	\begin{equation}
	I(z|\tau_{1:i-1}; \tau_{i}) = H(z|\tau_{1:i-1}) - H(z|\tau_{1:i}),
	\label{firstmi}
	\end{equation}
	where $\tau_{1:i-1}$ denotes the collected $i-1$ transitions before time step $i$. The above equation can also be interpreted as how much task belief $z$ has changed given the newly collected transition. We further transform the equation as follows:
	\begin{equation}
	\begin{aligned}
	I(z|\tau_{1:i-1}; \tau_{i})
	&= H(z|\tau_{1:i-1})+ H(z) - H(z) - H(z|\tau_{1:i})\\
	&= I(z; \tau_{1:i}) - I(z; \tau_{1:i-1})
	\end{aligned}
	\label{22}
	\end{equation}
	Equation~(\ref{22}) implies that the information gain from collecting new transition $\tau_{i}$ can be written as the temporal difference of the mutual information between task belief and collected trajectories. This indicate that we expect the exploration agent to collect informative experience that form a solid hypothesis for task $\mu_{n}$.
	
	Although theoretically sound, the mutual information of latent context and collected trajectories is hard to directly estimate. We expect a tractable form of Equation~(\ref{22}) without losing information. To this end, we first theoretically define the sufficient encoder for context in Meta-RL framework. Given two batches of transitions $b_{1}$ and $b_{2}$ from one task, the encoders $e_{1}(\cdot)$ and $e_{2}(\cdot)$ extract representations $c_{1} = e_{1}(b_{1})$ and $c_{2} = e_{2}(b_{2})$, respectively.
	
	\begin{Def}
		(Context-Sufficient Encoder) The context encoder $e(\cdot)$ of $b_{1}$ is sufficient if and only if $I(b_{1}, b_{2}) = I(e(b_{1}), b_{2})$.
	\end{Def}

	Intuitively, the mutual information remains the same as the encoder $e(\cdot)$ does not change the amount of information contained in $b_{1}$, which means the encoder is context-sufficient.
	
	Assuming the context encoder trained in previous section is context-sufficient, we can utilize this property to further transform the information gain objective in Equation (\ref{22}) in a view at the latent context level. Given batches of transitions from the same task $b_{pos}, \tau_{1:i} \sim \mu_{n}$ and a context-sufficient encoder $e(\cdot)$, the latent context $c_{1:i} = e(\tau_{1:i})$, $c_{1:i-1} = e(\tau_{1:i-1})$. As for the latent context $z$, we can approximate it as the embedding of positive transition batch as well: $z \approx e(b_{pos}) = c_{pos}$. Then (\ref{22}) becomes:
	\begin{equation}
	\begin{aligned}
	I(z|\tau_{1:i-1}; \tau_{i}) &= I(z; \tau_{1:i}) - I(z; \tau_{1:i-1})\\ &= I[z; e(\tau_{1:i})] - I[z; e(\tau_{1:i-1})]\\ &= I(c_{pos}; c_{1:i}) - I(c_{pos}; c_{1:i-1})
	\end{aligned}
	\label{diff2}
	\end{equation}
	Given this form of equation, we can further derive a tractable lower bound of $I(z|\tau_{1:i-1}; \tau_{i})$. The calculation process for the lower bound of Equation~(\ref{diff2}) can be decomposed into computing the lower bound of $I(c_{pos}; c_{1:i})$ and the upper bound of $I(c_{pos}; c_{1:i-1})$.
	
	As mentioned in Preliminaries, a commonly used lower bound of $I(c_{pos}, c_{1:i})$ can be written as:
	\begin{equation}
	I(c_{pos}; c_{1:i})\geqslant -L_{lower} + \log(W)
	\label{lowerbound1}
	\end{equation}
	where,
	\begin{equation}
	L_{lower} = -\mathop{\mathbb{E}}\limits_{C}\log\Big[\frac{\exp(f(c_{1:i}, c_{pos}))}{\sum _{j=1}^{M}\exp(f(c_{1:i}, c_{j}))} \Big]
	\label{lowerbound11}
	\end{equation}
	$W$ denotes the number of tasks. $C=C_{pos}\cup C_{neg}$, where $C_{pos}$ contains latent context from the same tasks while $C_{neg}$ contains latent context from different ones. As stated in previous section, we optimize the former term in (\ref{lowerbound1}) to train the context encoder and similarity function $f$.
	
	We also need to find a tractable form for the upper bound of $I[c_{pos}; c_{1:i-1}]$. Leveraging current contrastive learning framework, we make the following proposition:
	\begin{proposition}
		\begin{equation}
		\begin{aligned}
		I(c_{pos}; c_{1:i-1}) &\leqslant - L_{upper} + \log(W)
		\label{lowerbound2}
		\end{aligned}
		\end{equation}
		where,
		\begin{equation}
		\begin{aligned}
		L_{upper} = -\mathop{\mathbb{E}}\limits_{C_{pos}}log\Big[\frac{\exp(f(c_{1:i-1}, c_{pos}))}{\sum _{c_{j}\in C_{neg}}\exp(f(c_{1:i-1}, c_{j}))} \Big]
		\label{lowerbound3}
		\end{aligned}
		\end{equation}
	\end{proposition}
	\begin{proof}
		see Appendix.
	\end{proof}
	
	\noindent Compared with InfoNCE loss, the only difference between these two bounds is whether the similarity relations $exp(f(c_{1:i}; c_{pos}))$ between query and positive key is included in denominator. Then we can further derive the lower bound of Equation (\ref{diff2}) as:
	\begin{proposition}
		\begin{equation}
		\begin{aligned}
		&I(z|\tau_{1:i-1}; \tau_{i}) \geqslant L_{upper} - L_{lower} \\ &= \mathop{\mathbb{E}}\limits_{C_{pos}}\log\Big[\frac{\exp(f(c_{1:i}, c_{pos}))}{\sum _{c_{j}\in C}\exp(f(c_{1:i}, c_{j}))} \Big] -\mathop{\mathbb{E}}\limits_{C_{pos}}\log\Big[\frac{\exp(f(c_{1:i-1}, c_{pos}))}{\sum _{c_{j}\in C_{neg}}\exp(f(c_{1:i-1}, c_{j}))} \Big]
		\label{lowerbound5}
		\end{aligned}
		\end{equation}
	\end{proposition}
	\begin{proof}
		This can be derived by simply calculating the difference between Equation (\ref{lowerbound1}) and (\ref{lowerbound2}).
	\end{proof}
	Thus, we obtain an estimate of the lower bound of mutual information $I(z|c_{1:i-1}; c_{i})$ and we further use it as an intrinsic reward for the independent exploration agent:
	\begin{equation}
	\begin{aligned}
	r^{e} = r^{env} + \alpha r^{aux},\textit{where}~ r^{aux} = L_{upper} - L_{lower}
	\label{reward}
	\end{aligned}
	\end{equation}
	The intrinsic reward can be interpreted as the difference between two contrastive loss. It measures how much the inference for current task has been improved after collecting new transition $c_{i}$. To make this objective more clear and comprehensible, we can transform the lower bound (\ref{lowerbound5}) as: 
	\begin{equation}
	\begin{aligned}
	L_{upper} - L_{lower}  = &\mathop{\mathbb{E}}\limits_{C_{pos}}\Big[f(c_{1:i},c_{pos}) - f(c_{1:i-1},c_{pos}) \Big] \\ - &\mathop{\mathbb{E}}\limits_{C_{pos}}\log\Big[\frac{\sum _{c_{j}\in C}\exp(f(c_{1:i}, c_{j}))}{\sum _{c_{j}\in C_{neg}}\exp(f(c_{1:i-1}, c_{j}))}\Big]
	\label{lowerbound6}
	\end{aligned}
	\end{equation}

	Intuitively, the first term is an estimation of how much the similarity score for positive pairs (the correct task inference) has improved after collecting new transition.  Maximizing this term means that we want the exploration policy to collect transitions which make it easy to distinguish between tasks and make correct and solid task hypothesis. The second term can be interpreted as a regularization term. Maximizing this term means that while the agent tries to visit places which may result in enhancement of positive score, we want to limit the policy not visiting places where the negative score (the similarity with other tasks in embedding space) enhances as well.

	We summarize our training procedure in Figure~\ref{fig2}. The exploration agent and execution agent interact with environment separately to get their own replay buffer. The execution replay buffer also contains transitions from exploration agent that do not add on the intrinsic reward term. At the beginning of each meta-training episode, the two agents sample transition data $b_{exp}$ and $b_{exe}$ to train their policy separately. We independently sample from exploration buffer to obtain transition batches for calculating latent context. Note that the reward terms in $b_{enc}$ do not add on the intrinsic reward and only use the original environmental reward for computing latent context. In practice, we utilize Soft Actor-Critic~\citep{DBLP:conf/icml/HaarnojaZAL18} for both exploration and execution agent. We first pretrain the context encoder with contrastive loss for a few episodes to avoid the intrinsic reward being too noisy at the beginning of training. We also maintain an encoder replay buffer that only contain recently collected data for the same reason. During meta-testing, we first let the exploration agent collect trajectories for a few episodes then compute the latent context based on these experience. The execution agent then acts conditioned on the latent context. The pseudo-code for CCM during training and adaptation can be found in our appendix.

	\section{Experiments}
	In this section, we evaluate the performance of our proposed CCM to answer the following questions\footnote{We provide more experimental results in the Appendix, including ablation studies of the influence of intrinsic reward scale, context updating frequency and the regularization term respectively.}: 
	\begin{itemize}
		\item Does CCM's contrastive context encoder improve the performance of state-of-the-art context-based Meta-RL methods?
		\item After combining with information-gain-based exploration policy, does CCM improve the overall adaptation performance compared with prior Meta-RL algorithms? 
		\item Is CCM able to extract effective and reasonable context information?
	\end{itemize}

	\subsection{Comparison of Context Encoder Training Strategy}
	We first evaluate the performance of context-based Meta-RL methods after combining with contrastive context encoder on several continuous control tasks simulated via MuJoCo physics simulator~\citep{DBLP:conf/iros/TodorovET12}, which are standard Meta-RL benchmarks used in prior work~\citep{DBLP:conf/iclr/FakoorCSS20,DBLP:conf/icml/RakellyZFLQ19}. We also evaluate its performance on out-of-distribution (OOD) tasks, which are more challenging cases for testing encoder's capacity of extracting useful information. We compare to the existing two kinds of training strategy for context encoder: 1) RV (Recovering Value-function)~\citep{DBLP:conf/icml/RakellyZFLQ19,DBLP:conf/iclr/FakoorCSS20}, in which the encoder is trained with gradients from recovering the state-action value function. 2) DP (Dynamics Prediction)~\citep{DBLP:journals/corr/abs-2005-06800,DBLP:conf/iclr/ZhouPG19}, in which the encoder is trained by performing forward or backward prediction. Here, we follow the settings in CaDM~\citep{DBLP:journals/corr/abs-2005-06800} which uses the combined prediction loss of forward prediction and backward prediction to update the context encoder\footnote{For environments that have fixed dynamics, we further modify the model to predict the reward as well.}. Our method CCM+RV denotes the encoder receives gradients from both the contrastive loss and value function, while CCM+DP denotes the encoder receives gradients from both the contrastive loss and dynamics prediction loss. To maintain consistency, we does not consider extra exploration policy in this part of evaluation and all three methods use the same network structure (i.e. actor-critic, context encoder) and evaluation settings as state-of-the-art context-based Meta-RL method PEARL.

	The results are shown in Figure~\ref{fig5}. Both CCM+DP and CCM+RV achieve comparably or better performance than existing algorithms, implying that the proposed contrastive context encoder can help to extract contextual information more effectively. For in-distribution tasks that vary in dynamics (i.e. ant-mass and cheetah-mass), CCM+DP obtains better returns and converges faster than CCM+RV. This is consistent to what is empirically found in CaDM~\citep{DBLP:journals/corr/abs-2005-06800}: prediction models is more effective when the transition function changes across tasks. Moreover, recall that contrastive context encoder focus on the difference between tasks at trajectory level. Such property may compensate for DP's drawbacks which can embed commonly-shared dynamics knowledge that hinder the meta-learning process. We assume that this property leads to the better performance of CCM+DP on in-distribution tasks. 
	
	\begin{figure*}[t]
		\centering
		\includegraphics[width=0.98\textwidth]{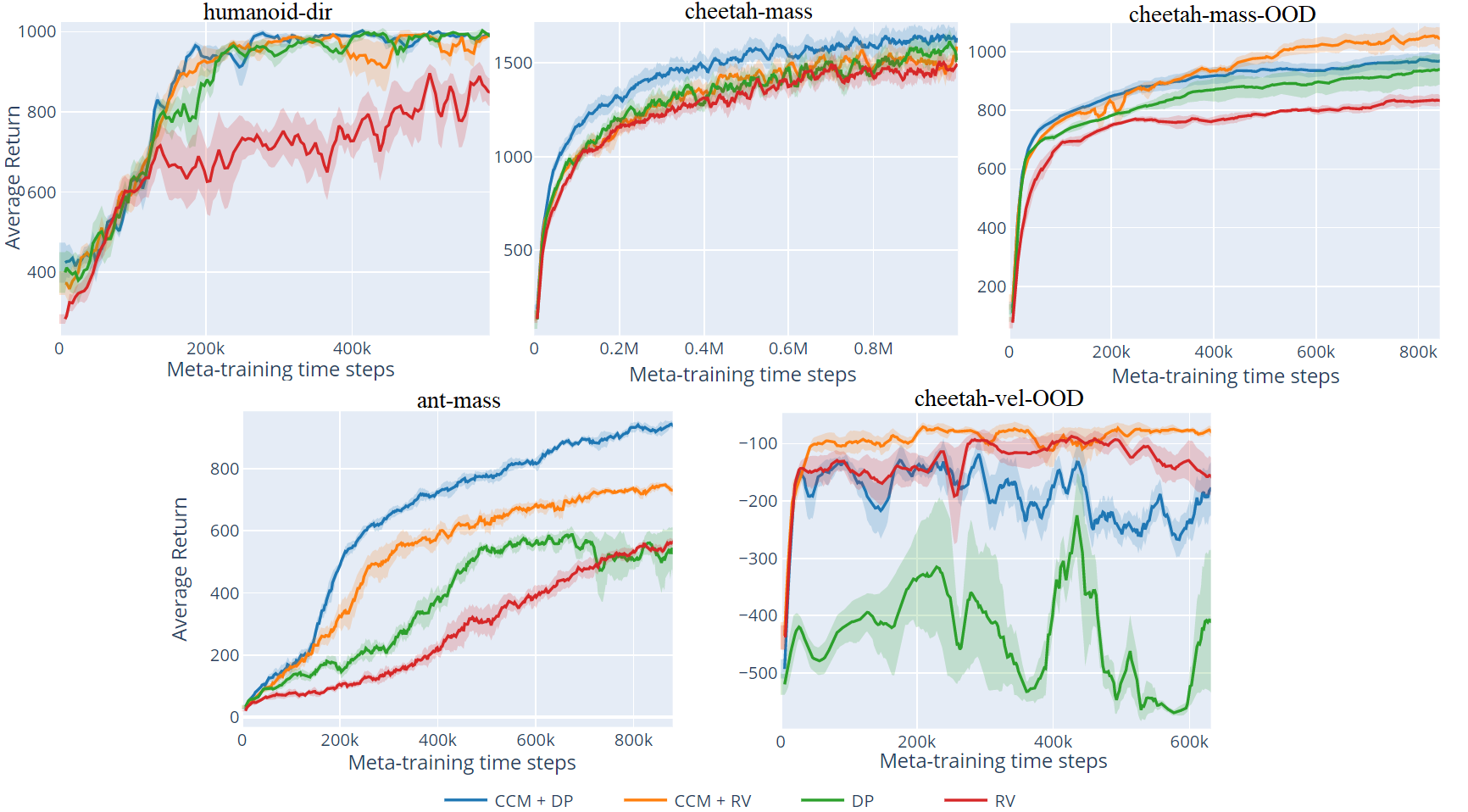}
		\caption{Comparison for different context encoder training strategies. Our methods CCM + DP and CCM + RV achieve consistently better performance compared to existing methods. The error bar shows 1 standard deviation.}
		\label{fig5}
	\end{figure*}
	
	For out-of-distribution tasks, CCM+RV outperforms other methods including CCM+DP, implying that contrastive loss combined with loss from recovering value function obtains better generalization for different tasks. Recovering value function focuses on extracting high-level task-specific features. This property avoids the overfitting problem of dynamics prediction (limited to common dynamics in training tasks), of which is amplified in this case. However, solely using RV to train the encoder can perform poorly due to its noisy updating signal as well as the ignorance of detailed transition difference that contain task-specific information. Combining with CCM addresses such problems through the usage of low-variance InfoNCE loss as well as forcing the encoder explicitly find the trajectory-level information that varies tasks from each other.

	\subsection{Comparison of Overall Adaptation Performance on Complex Environments}
	We then consider both components of CCM algorithm and evaluate on several complex sparse-reward environments to compare the performance of CCM with prior Meta-RL algorithms and show whether the information-gain-based exploration strategy improve its adaptation performance. For cheetah-sparse and walker-sparse environments, agent receives a reward when it is close to the target velocity. In hard-point-robot, agent needs to reach two randomly selected goals one after another, and receives a reward when it's close enough to the goals. We compare CCM with state-of-the-art Meta-RL approach PEARL, and we further modify it with a contrastive context encoder for a fair comparison (PEARL-CL) to measure the influence of CCM's exploration strategy. We further consider MAML~\citep{DBLP:conf/icml/FinnAL17} and ProMP~\citep{DBLP:conf/iclr/RothfussLCAA19}, which adapt with on-policy methods. We also compare to varibad~\citep{DBLP:conf/iclr/ZintgrafSISGHW20}, which is a stable version of RL2~\citep{DBLP:journals/corr/DuanSCBSA16,DBLP:journals/corr/WangKTSLMBKB16}. 
	
	As shown in Figure~\ref{fig6}, CCM achieves better performance than prior Meta-RL algorithms in both final returns and learning speed in such complex sparse-reward environments. Within relatively small number of environment interactions, on policy methods MAML, varibad and ProMP struggle to learn effective policies on these complex environment while off-policy context-based methods CCM and PEARL-CL generally achieves higher performance.
	
	\begin{figure*}[t]
		\centering
		\includegraphics[width=0.95\textwidth]{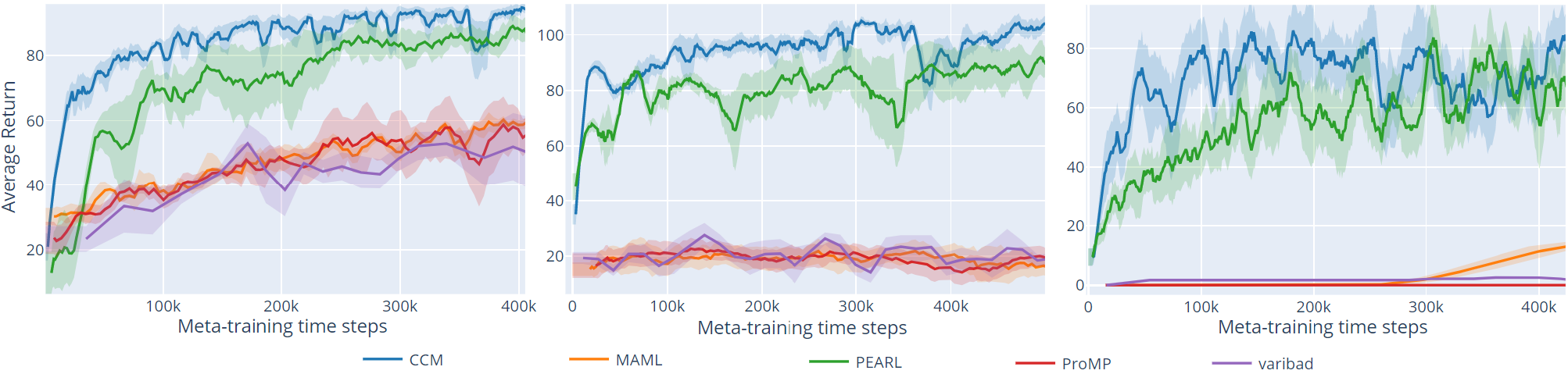}
		
		\caption{CCM's overall performance compared with state-of-the-art Meta-RL methods on complex sparse-reward environments. From left to right:
			\textit{walker-sparse},\textit{cheetah-sparse},  \textit{hard-point-robot}.The error bar shows 1 standard deviation.}
		\label{fig6}
	\end{figure*}
	
	As the only difference between CCM and PEARL-CL here is whether to use the proposed extra exploration policy, we can conclude that CCM's information-gain-based exploration strategy enables fast adaptation by collecting more informative experience and further improving the quality of context when carrying out execution policies. Note that during training phases, trajectories collected by the exploration agent are used for updating the context encoder and the execution policy as well. As shown in experimental results, this may lead to faster training speed of context encoder and execution policy because of the more informative replay buffer.

	\subsection{Visualization for Context}
	Finally, we visualize CCM's learned context during adaptation phases via t-SNE~\citep{Maaten2008VisualizingDU} and compare with PEARL. We run the learned policies on ten randomly sampled test tasks multiple times to collect trajectories. Further, we encode the collected trajectories into latent context in embedding space with the learned context encoder and visualize via t-SNE in Figure~\ref{fig7}. We find that the latent context generated by CCM from the same tasks is close together in embedding space while maintains clear boundaries between different tasks. In contrast, the latent context generated by PEARL shows globally messy clustering results with only a few reasonable patterns in local region. This indicates that CCM extracts high-quality (i.e. compact and sufficient) task-specific information from the environment compared with PEARL. As a result, the policy conditioned on the high-quality latent context is more likely to get a higher return on those meta-testing tasks, which is consistent to our prior empirical results. 
	
	\begin{figure}[htbp]
		\centering
		
		\includegraphics[width=0.65\textwidth]{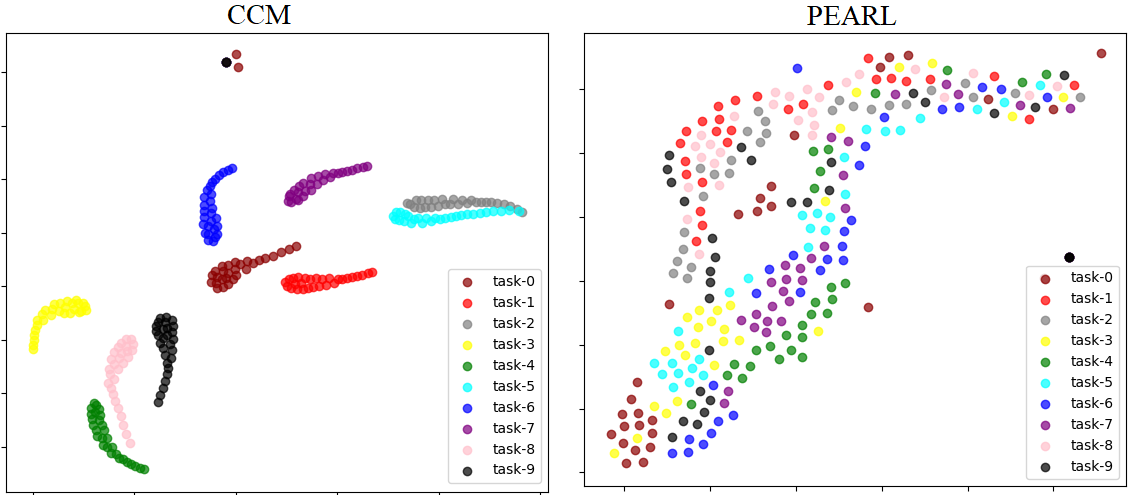}
		\caption{Visualization of context in embedding space. Different color represents context from different tasks.}
		\label{fig7}
	\end{figure}
	
	\section{Related Work}
	\subsection{Contrastive Learning}
	Contrastive learning has recently achieved great success in learning representations for image or sequential data~\cite{DBLP:journals/corr/abs-1807-03748,DBLP:journals/corr/abs-1905-09272,DBLP:conf/cvpr/He0WXG20,DBLP:journals/corr/abs-2002-05709}. In RL, it has been used to extract reward signals in latent space~\cite{DBLP:conf/icra/SermanetLCHJSLB18,DBLP:conf/iros/DwibediTLS18}, or used as an auxiliary task to study representations for high-dimensional data~\cite{DBLP:journals/corr/abs-2004-04136,DBLP:conf/nips/AnandROBCH19}. Contrastive learning helps learn representations that obey similarity constraints by dividing the data set into similar (positive) and dissimilar (negative) pairs and minimizes contrastive loss. Prior work~\cite{DBLP:journals/corr/abs-2004-04136,DBLP:journals/corr/abs-1905-09272} has shown various methods of generating positive and negative pairs for image-based input data. The standard approach is to create multiple views of each datapoint like random crops and data augmentations~\cite{Wu2018UnsupervisedFL,DBLP:journals/corr/abs-2002-05709,DBLP:conf/cvpr/He0WXG20}. However, in this work we focus on low dimensional input data and leverage natural discrimination inside the trajectories of different tasks to generate positive and negative data. The selection of contrastive loss function is also various and the most competitive one is InfoNCE~\cite{DBLP:journals/corr/abs-1807-03748}. The motivation behind contrastive loss is the InfoMax principle~\cite{DBLP:journals/computer/Linsker88}, which can be interpreted as maximizing the mutual information between two views of data. The relationships between InfoNCE loss and mutual information is conprehensively explained in~\cite{DBLP:conf/icml/PooleOOAT19}.
	
	\subsection{Meta-Reinforcement Learning}
	Meta-RL extends the framework of meta-learning~\cite{Schmidhuber1987EvolutionaryPI,Thrun1998LearningTL,Naik1992MetaneuralNT} to the reinforcement learning setting. Recurrent or recursive Meta-RL methods~\cite{DBLP:journals/corr/DuanSCBSA16,DBLP:journals/corr/WangKTSLMBKB16} directly learn a function that takes in past experiences and outputs a policy for the agent. Gradient-based Meta-RL~\cite{DBLP:conf/icml/FinnAL17,DBLP:conf/iclr/RothfussLCAA19,Liu2019TamingME,Gupta2018MetaReinforcementLO} methods meta learn a model initialization and adapt the parameters with policy gradient methods. These two kinds of methods focus on on-policy meta-training and are struggling to achieve good performance on complicated environments.
	
	We here consider context-based Meta-RL~\cite{DBLP:conf/icml/RakellyZFLQ19,Fu2019EfficientMR, DBLP:conf/iclr/FakoorCSS20}, which is one kind of Meta-RL algorithms that is able to meta-learn from off-policy data by leveraging context. Rakelly et al.~(\citeyear{DBLP:conf/icml/RakellyZFLQ19}) propose PEARL that adapts to a new environment by inferring latent context variables from a small number of trajectories. Fakoor et al.~(\citeyear{DBLP:conf/iclr/FakoorCSS20}) propose MQL showing that context combined with other vanilla RL algorithms performs comparably to PEARL. Lee et al.~(\citeyear{DBLP:journals/corr/abs-2005-06800}) learn a global model that generalizes across tasks by training a latent context to capture the local dynamics. In contrast, our approach directly focuses on the context itself, which motivates an algorithm to improve the quality of latent context.
	
	Some prior research has proposed to explore to obtain the most informative trajectories in Meta-RL~\cite{DBLP:conf/icml/RakellyZFLQ19,DBLP:conf/iclr/RothfussLCAA19,Gupta2018MetaReinforcementLO}. The common idea behind these approaches is Thompson-sampling~\cite{Schmidhuber1987EvolutionaryPI}. However, this kind of posterior sampling is conditioned on learned execution policy which may lead to the exploration behavior only limited to the learned tasks. In contrast, we theoretically obtain a lower bound estimation of the exploration objective, based on which a separate exploration agent is trained.
	
	\section{Conclusion}
	In this paper, we propose that constructing a powerful context for Meta-RL involves two problems: 1) How to collect informative trajectories of which the corresponding context reflects the specification of tasks? 2) How to train a compact and sufficient encoder that can embed the task-specific information contained in prior trajectories? We then propose our method CCM which tackles the above two problems respectively. Firstly, CCM focuses on the underlying structure behind different tasks' transitions and trains the encoder by leveraging contrastive learning. CCM further learns a separate exploration agent with an information-theoretical objective that aims to maximize the improvement of inference after collecting new transitions. The empirical results on several complex simulated control tasks show that CCM outperforms state-of-the-art Meta-RL methods by addressing the aforementioned problems. 
	
	\section{Acknowledgements}
	We thank Chenjia Bai, Qianli Shen, and Weixun Wang for useful discussions and suggestions.
	
	\bibliography{aaai21.bib}

\begin{thebibliography}{36}
\providecommand{\natexlab}[1]{#1}
\providecommand{\url}[1]{\texttt{#1}}
\providecommand{\urlprefix}{URL }
\expandafter\ifx\csname urlstyle\endcsname\relax
  \providecommand{\doi}[1]{doi:\discretionary{}{}{}#1}\else
  \providecommand{\doi}{doi:\discretionary{}{}{}\begingroup
  \urlstyle{rm}\Url}\fi

\bibitem[{Anand et~al.(2019)Anand, Racah, Ozair, Bengio, C{\^{o}}t{\'{e}}, and
  Hjelm}]{DBLP:conf/nips/AnandROBCH19}
Anand, A.; Racah, E.; Ozair, S.; Bengio, Y.; C{\^{o}}t{\'{e}}, M.; and Hjelm,
  R.~D. 2019.
\newblock Unsupervised State Representation Learning in Atari.
\newblock In \emph{Advances in Neural Information Processing Systems 32: Annual
  Conference on Neural Information Processing Systems 2019, NeurIPS 2019},
  8766--8779.

\bibitem[{Brockman et~al.(2016)Brockman, Cheung, Pettersson, Schneider,
  Schulman, Tang, and Zaremba}]{DBLP:journals/corr/BrockmanCPSSTZ16}
Brockman, G.; Cheung, V.; Pettersson, L.; Schneider, J.; Schulman, J.; Tang,
  J.; and Zaremba, W. 2016.
\newblock OpenAI Gym.
\newblock \emph{CoRR} abs/1606.01540.

\bibitem[{Chen et~al.(2020)Chen, Kornblith, Norouzi, and
  Hinton}]{DBLP:journals/corr/abs-2002-05709}
Chen, T.; Kornblith, S.; Norouzi, M.; and Hinton, G.~E. 2020.
\newblock A Simple Framework for Contrastive Learning of Visual
  Representations.
\newblock \emph{CoRR} abs/2002.05709.

\bibitem[{Duan et~al.(2016)Duan, Schulman, Chen, Bartlett, Sutskever, and
  Abbeel}]{DBLP:journals/corr/DuanSCBSA16}
Duan, Y.; Schulman, J.; Chen, X.; Bartlett, P.~L.; Sutskever, I.; and Abbeel,
  P. 2016.
\newblock RL{\textdollar}{\^{}}2{\textdollar}: Fast Reinforcement Learning via
  Slow Reinforcement Learning.
\newblock \emph{CoRR} abs/1611.02779.

\bibitem[{Dwibedi et~al.(2018)Dwibedi, Tompson, Lynch, and
  Sermanet}]{DBLP:conf/iros/DwibediTLS18}
Dwibedi, D.; Tompson, J.; Lynch, C.; and Sermanet, P. 2018.
\newblock Learning Actionable Representations from Visual Observations.
\newblock In \emph{2018 {IEEE/RSJ} International Conference on Intelligent
  Robots and Systems, {IROS} 2018}, 1577--1584. {IEEE}.

\bibitem[{Fakoor et~al.(2020)Fakoor, Chaudhari, Soatto, and
  Smola}]{DBLP:conf/iclr/FakoorCSS20}
Fakoor, R.; Chaudhari, P.; Soatto, S.; and Smola, A.~J. 2020.
\newblock Meta-Q-Learning.
\newblock In \emph{8th International Conference on Learning Representations,
  {ICLR} 2020}. OpenReview.net.

\bibitem[{Finn, Abbeel, and Levine(2017)}]{DBLP:conf/icml/FinnAL17}
Finn, C.; Abbeel, P.; and Levine, S. 2017.
\newblock Model-Agnostic Meta-Learning for Fast Adaptation of Deep Networks.
\newblock In \emph{Proceedings of the 34th International Conference on Machine
  Learning, {ICML} 2017}, 1126--1135.

\bibitem[{Fu et~al.(2019)Fu, Tang, Hao, Liu, and Chen}]{Fu2019EfficientMR}
Fu, H.; Tang, H.; Hao, J.; Liu, W.; and Chen, C. 2019.
\newblock MGHRL: Meta Goal-generation for Hierarchical Reinforcement Learning.
\newblock \emph{ArXiv} abs/1909.13607.

\bibitem[{Gupta et~al.(2018)Gupta, Mendonca, Liu, Abbeel, and
  Levine}]{Gupta2018MetaReinforcementLO}
Gupta, A.; Mendonca, R.; Liu, Y.; Abbeel, P.; and Levine, S. 2018.
\newblock Meta-Reinforcement Learning of Structured Exploration Strategies.
\newblock In \emph{NeurIPS}.

\bibitem[{Haarnoja et~al.(2018)Haarnoja, Zhou, Abbeel, and
  Levine}]{DBLP:conf/icml/HaarnojaZAL18}
Haarnoja, T.; Zhou, A.; Abbeel, P.; and Levine, S. 2018.
\newblock Soft Actor-Critic: Off-Policy Maximum Entropy Deep Reinforcement
  Learning with a Stochastic Actor.
\newblock In \emph{Proceedings of the 35th International Conference on Machine
  Learning, {ICML} 2018}, 1856--1865.

\bibitem[{He et~al.(2020)He, Fan, Wu, Xie, and
  Girshick}]{DBLP:conf/cvpr/He0WXG20}
He, K.; Fan, H.; Wu, Y.; Xie, S.; and Girshick, R.~B. 2020.
\newblock Momentum Contrast for Unsupervised Visual Representation Learning.
\newblock In \emph{2020 {IEEE/CVF} Conference on Computer Vision and Pattern
  Recognition, {CVPR} 2020}, 9726--9735. {IEEE}.

\bibitem[{H{\'{e}}naff et~al.(2019)H{\'{e}}naff, Srinivas, Fauw, Razavi,
  Doersch, Eslami, and van~den Oord}]{DBLP:journals/corr/abs-1905-09272}
H{\'{e}}naff, O.~J.; Srinivas, A.; Fauw, J.~D.; Razavi, A.; Doersch, C.;
  Eslami, S. M.~A.; and van~den Oord, A. 2019.
\newblock Data-Efficient Image Recognition with Contrastive Predictive Coding.
\newblock \emph{CoRR} abs/1905.09272.

\bibitem[{Lee et~al.(2020)Lee, Seo, Lee, Lee, and
  Shin}]{DBLP:journals/corr/abs-2005-06800}
Lee, K.; Seo, Y.; Lee, S.; Lee, H.; and Shin, J. 2020.
\newblock Context-aware Dynamics Model for Generalization in Model-Based
  Reinforcement Learning.
\newblock \emph{CoRR} abs/2005.06800.

\bibitem[{Lillicrap et~al.(2016)Lillicrap, Hunt, Pritzel, Heess, Erez, Tassa,
  Silver, and Wierstra}]{DBLP:journals/corr/LillicrapHPHETS15}
Lillicrap, T.~P.; Hunt, J.~J.; Pritzel, A.; Heess, N.; Erez, T.; Tassa, Y.;
  Silver, D.; and Wierstra, D. 2016.
\newblock Continuous control with deep reinforcement learning.
\newblock In \emph{4th International Conference on Learning Representations,
  {ICLR} 2016}.

\bibitem[{Linsker(1988)}]{DBLP:journals/computer/Linsker88}
Linsker, R. 1988.
\newblock Self-Organization in a Perceptual Network.
\newblock \emph{Computer} 21(3): 105--117.

\bibitem[{Liu et~al.(2020)Liu, Raghunathan, Liang, and
  Finn}]{DBLP:journals/corr/abs-2008-02790}
Liu, E.~Z.; Raghunathan, A.; Liang, P.; and Finn, C. 2020.
\newblock Explore then Execute: Adapting without Rewards via Factorized
  Meta-Reinforcement Learning.
\newblock \emph{CoRR} abs/2008.02790.

\bibitem[{Liu, Socher, and Xiong(2019)}]{Liu2019TamingME}
Liu, H.; Socher, R.; and Xiong, C. 2019.
\newblock Taming MAML: Efficient unbiased meta-reinforcement learning.
\newblock In \emph{ICML}.

\bibitem[{Maaten and Hinton(2008)}]{Maaten2008VisualizingDU}
Maaten, L. V.~D.; and Hinton, G.~E. 2008.
\newblock Visualizing Data using t-SNE.
\newblock \emph{Journal of Machine Learning Research} 9: 2579--2605.

\bibitem[{Mnih et~al.(2015)Mnih, Kavukcuoglu, Silver, Rusu, Veness, Bellemare,
  Graves, Riedmiller, Fidjeland, Ostrovski, Petersen, Beattie, Sadik,
  Antonoglou, King, Kumaran, Wierstra, Legg, and
  Hassabis}]{DBLP:journals/nature/MnihKSRVBGRFOPB15}
Mnih, V.; Kavukcuoglu, K.; Silver, D.; Rusu, A.~A.; Veness, J.; Bellemare,
  M.~G.; Graves, A.; Riedmiller, M.~A.; Fidjeland, A.; Ostrovski, G.; Petersen,
  S.; Beattie, C.; Sadik, A.; Antonoglou, I.; King, H.; Kumaran, D.; Wierstra,
  D.; Legg, S.; and Hassabis, D. 2015.
\newblock Human-level control through deep reinforcement learning.
\newblock \emph{Nature} 518(7540): 529--533.
\newblock \doi{10.1038/nature14236}.

\bibitem[{Naik and Mammone(1992)}]{Naik1992MetaneuralNT}
Naik, D.~K.; and Mammone, R. 1992.
\newblock Meta-neural networks that learn by learning.
\newblock \emph{[Proceedings 1992] IJCNN International Joint Conference on
  Neural Networks} 1: 437--442 vol.1.

\bibitem[{Pathak et~al.(2017)Pathak, Agrawal, Efros, and
  Darrell}]{DBLP:conf/icml/PathakAED17}
Pathak, D.; Agrawal, P.; Efros, A.~A.; and Darrell, T. 2017.
\newblock Curiosity-driven Exploration by Self-supervised Prediction.
\newblock In Precup, D.; and Teh, Y.~W., eds., \emph{Proceedings of the 34th
  International Conference on Machine Learning, {ICML} 2017}, volume~70 of
  \emph{Proceedings of Machine Learning Research}, 2778--2787. {PMLR}.

\bibitem[{Poole et~al.(2019)Poole, Ozair, van~den Oord, Alemi, and
  Tucker}]{DBLP:conf/icml/PooleOOAT19}
Poole, B.; Ozair, S.; van~den Oord, A.; Alemi, A.; and Tucker, G. 2019.
\newblock On Variational Bounds of Mutual Information.
\newblock In Chaudhuri, K.; and Salakhutdinov, R., eds., \emph{Proceedings of
  the 36th International Conference on Machine Learning, {ICML} 2019},
  volume~97 of \emph{Proceedings of Machine Learning Research}, 5171--5180.
  {PMLR}.

\bibitem[{Rakelly et~al.(2019)Rakelly, Zhou, Finn, Levine, and
  Quillen}]{DBLP:conf/icml/RakellyZFLQ19}
Rakelly, K.; Zhou, A.; Finn, C.; Levine, S.; and Quillen, D. 2019.
\newblock Efficient Off-Policy Meta-Reinforcement Learning via Probabilistic
  Context Variables.
\newblock In \emph{Proceedings of the 36th International Conference on Machine
  Learning, {ICML} 2019}, 5331--5340.

\bibitem[{Rothfuss et~al.(2019)Rothfuss, Lee, Clavera, Asfour, and
  Abbeel}]{DBLP:conf/iclr/RothfussLCAA19}
Rothfuss, J.; Lee, D.; Clavera, I.; Asfour, T.; and Abbeel, P. 2019.
\newblock ProMP: Proximal Meta-Policy Search.
\newblock In \emph{7th International Conference on Learning Representations,
  {ICLR} 2019}.

\bibitem[{Schmidhuber(1987)}]{Schmidhuber1987EvolutionaryPI}
Schmidhuber, J. 1987.
\newblock Evolutionary principles in self-referential learning.

\bibitem[{Schulman et~al.(2015)Schulman, Levine, Abbeel, Jordan, and
  Moritz}]{DBLP:conf/icml/SchulmanLAJM15}
Schulman, J.; Levine, S.; Abbeel, P.; Jordan, M.~I.; and Moritz, P. 2015.
\newblock Trust Region Policy Optimization.
\newblock In \emph{Proceedings of the 32nd International Conference on Machine
  Learning, {ICML} 2015}, 1889--1897.

\bibitem[{Sermanet et~al.(2018)Sermanet, Lynch, Chebotar, Hsu, Jang, Schaal,
  and Levine}]{DBLP:conf/icra/SermanetLCHJSLB18}
Sermanet, P.; Lynch, C.; Chebotar, Y.; Hsu, J.; Jang, E.; Schaal, S.; and
  Levine, S. 2018.
\newblock Time-Contrastive Networks: Self-Supervised Learning from Video.
\newblock In \emph{2018 {IEEE} International Conference on Robotics and
  Automation, {ICRA} 2018}, 1134--1141. {IEEE}.

\bibitem[{Srinivas, Laskin, and
  Abbeel(2020)}]{DBLP:journals/corr/abs-2004-04136}
Srinivas, A.; Laskin, M.; and Abbeel, P. 2020.
\newblock {CURL:} Contrastive Unsupervised Representations for Reinforcement
  Learning.
\newblock \emph{CoRR} abs/2004.04136.

\bibitem[{Thrun and Pratt(1998)}]{Thrun1998LearningTL}
Thrun, S.; and Pratt, L.~Y. 1998.
\newblock Learning to Learn.
\newblock In \emph{Springer US}.

\bibitem[{Todorov, Erez, and Tassa(2012)}]{DBLP:conf/iros/TodorovET12}
Todorov, E.; Erez, T.; and Tassa, Y. 2012.
\newblock MuJoCo: {A} physics engine for model-based control.
\newblock In \emph{2012 {IEEE/RSJ} International Conference on Intelligent
  Robots and Systems, {IROS} 2012}, 5026--5033.

\bibitem[{van~den Oord, Li, and
  Vinyals(2018)}]{DBLP:journals/corr/abs-1807-03748}
van~den Oord, A.; Li, Y.; and Vinyals, O. 2018.
\newblock Representation Learning with Contrastive Predictive Coding.
\newblock \emph{CoRR} abs/1807.03748.

\bibitem[{Wang et~al.(2016)Wang, Kurth{-}Nelson, Tirumala, Soyer, Leibo, Munos,
  Blundell, Kumaran, and Botvinick}]{DBLP:journals/corr/WangKTSLMBKB16}
Wang, J.~X.; Kurth{-}Nelson, Z.; Tirumala, D.; Soyer, H.; Leibo, J.~Z.; Munos,
  R.; Blundell, C.; Kumaran, D.; and Botvinick, M. 2016.
\newblock Learning to reinforcement learn.
\newblock \emph{CoRR} abs/1611.05763.

\bibitem[{Wu et~al.(2018)Wu, Xiong, Yu, and Lin}]{Wu2018UnsupervisedFL}
Wu, Z.; Xiong, Y.; Yu, S.; and Lin, D. 2018.
\newblock Unsupervised Feature Learning via Non-parametric Instance
  Discrimination.
\newblock \emph{2018 IEEE/CVF Conference on Computer Vision and Pattern
  Recognition} 3733--3742.

\bibitem[{Zhang et~al.(2020)Zhang, Wang, Hu, Chen, Fan, and
  Zhang}]{DBLP:journals/corr/abs-2006-08170}
Zhang, J.; Wang, J.; Hu, H.; Chen, Y.; Fan, C.; and Zhang, C. 2020.
\newblock Learn to Effectively Explore in Context-Based Meta-RL.
\newblock \emph{CoRR} abs/2006.08170.

\bibitem[{Zhou, Pinto, and Gupta(2019)}]{DBLP:conf/iclr/ZhouPG19}
Zhou, W.; Pinto, L.; and Gupta, A. 2019.
\newblock Environment Probing Interaction Policies.
\newblock In \emph{7th International Conference on Learning Representations,
  {ICLR} 2019}. OpenReview.net.

\bibitem[{Zintgraf et~al.(2020)Zintgraf, Shiarlis, Igl, Schulze, Gal, Hofmann,
  and Whiteson}]{DBLP:conf/iclr/ZintgrafSISGHW20}
Zintgraf, L.~M.; Shiarlis, K.; Igl, M.; Schulze, S.; Gal, Y.; Hofmann, K.; and
  Whiteson, S. 2020.
\newblock VariBAD: {A} Very Good Method for Bayes-Adaptive Deep {RL} via
  Meta-Learning.
\newblock In \emph{8th International Conference on Learning Representations,
  {ICLR} 2020}. OpenReview.net.

\end{thebibliography}
	\bibliographystyle{iclr2020_conference}
	
	\clearpage
	
	\onecolumn
	
	\appendix
	\section{Environment Details}
	We run all experiments with OpenAI gym~\citep{DBLP:journals/corr/BrockmanCPSSTZ16}, with the mujoco simulator~\citep{DBLP:conf/iros/TodorovET12}. The benchmarks used in our experiments are visualized in Figure~\ref{figapp0}. We further modify the original tasks to be Meta-RL tasks similar to~\citep{DBLP:conf/icml/RakellyZFLQ19,DBLP:journals/corr/abs-2005-06800,DBLP:conf/iclr/FakoorCSS20}:
	\begin{itemize}
		\item humanoid-dir: The target direction of running changes across tasks. The horizon is 200.
		\item cheetah-mass: The mass of cheetah changes across tasks to change transition dynamics. The horizon is 200.
		\item cheetah-mass-OOD: The mass of cheetah for a training task is sampled uniformly from $[0.2, 1.5]$ while that for test task is sampled uniformly randomly from $[1.5, 1.8]$. The horizon is 200.
		\item ant-mass: The mass of ant changes across tasks to change transition dynamics. The horizon is 200.
		\item cheetah-vel-OOD: Target velocity for a training task is sampled uniformly from $[0, 2.5]$ while that for test task is sampled uniformly randomly from $[2.5, 3.0]$. The horizon is 200.
		\item cheetah-sparse: Both target velocity and mass of cheetah change across tasks. The agent receives a dense reward only when it is close enough to the target velocity. The horizon is 64.
		\item walker-sparse: The target velocity changes across tasks. The agent receives a dense reward only when it is close enough to the target velocity. The horizon is 64.
		\item hard-point-robot: The agent needs to reach two randomly selected goals one after another with reward given when inside the goal radius. The horizon is 40.
	\end{itemize}
	
	\begin{figure}[htbp]
		\centering
		
		\includegraphics[width=0.9\textwidth]{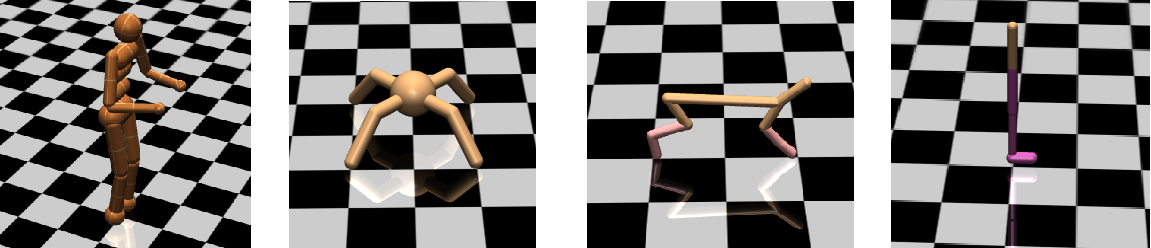}
		\caption{Meta-RL tasks: left-to-right: \textit{humanoid, ant, half cheetah}, and \textit{walker}.}
		\label{figapp0}
	\end{figure}
	
	\begin{algorithm}[htb]
		\caption{CCM Meta-training}
		\label{alg1}
		\begin{algorithmic}[1]
			\REQUIRE Batch of training tasks $\{\mu_{n}\}_{n=1,...,M}$ from $p(\mu)$,
			
			\STATE Initialize execution replay buffer $\mathcal{B}^{i}_{exe}$, exploration replay buffer $\mathcal{B}^{i}_{exp}$ for each training task
			\STATE Initialize parameters $\theta_{exp}$ and $\theta_{exe}$ for the off-policy method employed by exploration and execution respectively, e.g., SAC
			\STATE Initialize parameters $\theta_{enc}$ for context encoder network
			\WHILE{not done}
			\FOR{each task $\mu_{n}$}
			\STATE Roll-out exploration policy $\pi_{exp}$, producing transitions $\{(s_{j},a_{j},s'_{j},r^{env}_{j},r^{exp}_{j})\}_{j:1 \cdots N}$, where $r^{exp}_{j} = r^{env}_{j} + \alpha r^{aux}_{j}$
			\STATE Add tuples to exploration replay buffer $\mathcal{B}^{i}_{exp}$ and execution replay buffer $\mathcal{B}^{i}_{exe}$
			\STATE Roll-out execution policy $\pi_{exe}$, producing transitions $\{(s_{k},a_{k},s'_{k},r^{env}_{k})\}_{k:1 \cdots K}$
			\STATE Add tuples to execution replay buffer $\mathcal{B}^{i}_{exe}$
			
			\ENDFOR
			\FOR{each training step}
			\FOR{each task $\mu_{n}$}
			\STATE Sample transitions for encoder $b^{n}_{enc} = \{(s_{i},a_{i},s'_{i},r^{env}_{i})\}_{i:1 \cdots T}\sim \mathcal{B}^{n}_{exp}$ and RL batch \\ $b^{n}_{exp} = \{(s_{i},a_{i},s'_{i},r^{exp}_{i})\}_{i:1 \cdots N'} \sim\mathcal{B}^{n}_{exp}$, $b^{n}_{exe} = \{(s_{i},a_{i},s'_{i},r^{env}_{i})\}_{i:1 \cdots K'} \sim\mathcal{B}^{n}_{exe}$
			\STATE Update $\theta_{enc}$ and $\theta_{exe}$ with RL loss and contrastive loss using $b^{n}_{enc},b^{n}_{exe}$
			\STATE Update $\theta_{exp}$ with RL loss using $b^{n}_{enc},b^{n}_{exp}$
			\ENDFOR
			\ENDFOR
			\ENDWHILE
		\end{algorithmic}
	\end{algorithm}
	
	\begin{algorithm}[htb]
		\caption{CCM Meta-testing}
		\label{alg2}
		\begin{algorithmic}[1]
			\REQUIRE Test task $\mu \sim p(\mu)$,
			\STATE Initialize $b^{\mu} = \{\}$ 
			
			\FOR{each episode}
			\FOR{$k = 1,\cdots,K$}
			\STATE Roll-out exploration policy $\pi_{exp}$, producing transition $D_{exp}^{\mu} = \{(s_{j},a_{j},s'_{j},r^{env}_{j})\}$
			\STATE Accumulate transitions $b^{\mu} = b^{\mu} \cap D_{exp}^{\mu}$
			\ENDFOR
			\ENDFOR
			\FOR{$i = 1,\cdots,K$}
			\STATE Roll-out execution policy $\pi_{exe}$ conditioned on $b^{\mu}$, producing transition $D_{exe}^{\mu} = \{(s_{j},a_{j},s'_{j},r^{env}_{j})\}$
			\STATE Accumulate transitions $b^{\mu} = b^{\mu} \cap D_{exe}^{\mu}$
			\ENDFOR
			
		\end{algorithmic}
		
	\end{algorithm}

	\section{Implementation Details}
	For each environment, we meta-train 3 models, and meta-test each of them. The evaluation results reflect the average return in the last episode over a test rollout. We show the pseudo-code for CCM during meta-training and meta-testing in Algorithm~\ref{alg1} and \ref{alg2}. Note that here we calculate the exploration intrinsic reward during interaction with environment instead of training to reduce the computational complexity. To avoid the intrinsic reward being to noisy, we pretrain the context encoder for several episodes and set the exploration replay buffer to only contain recently collected data.
	\begin{figure}[htb]
		\centering
		\includegraphics[width=0.5\textwidth]{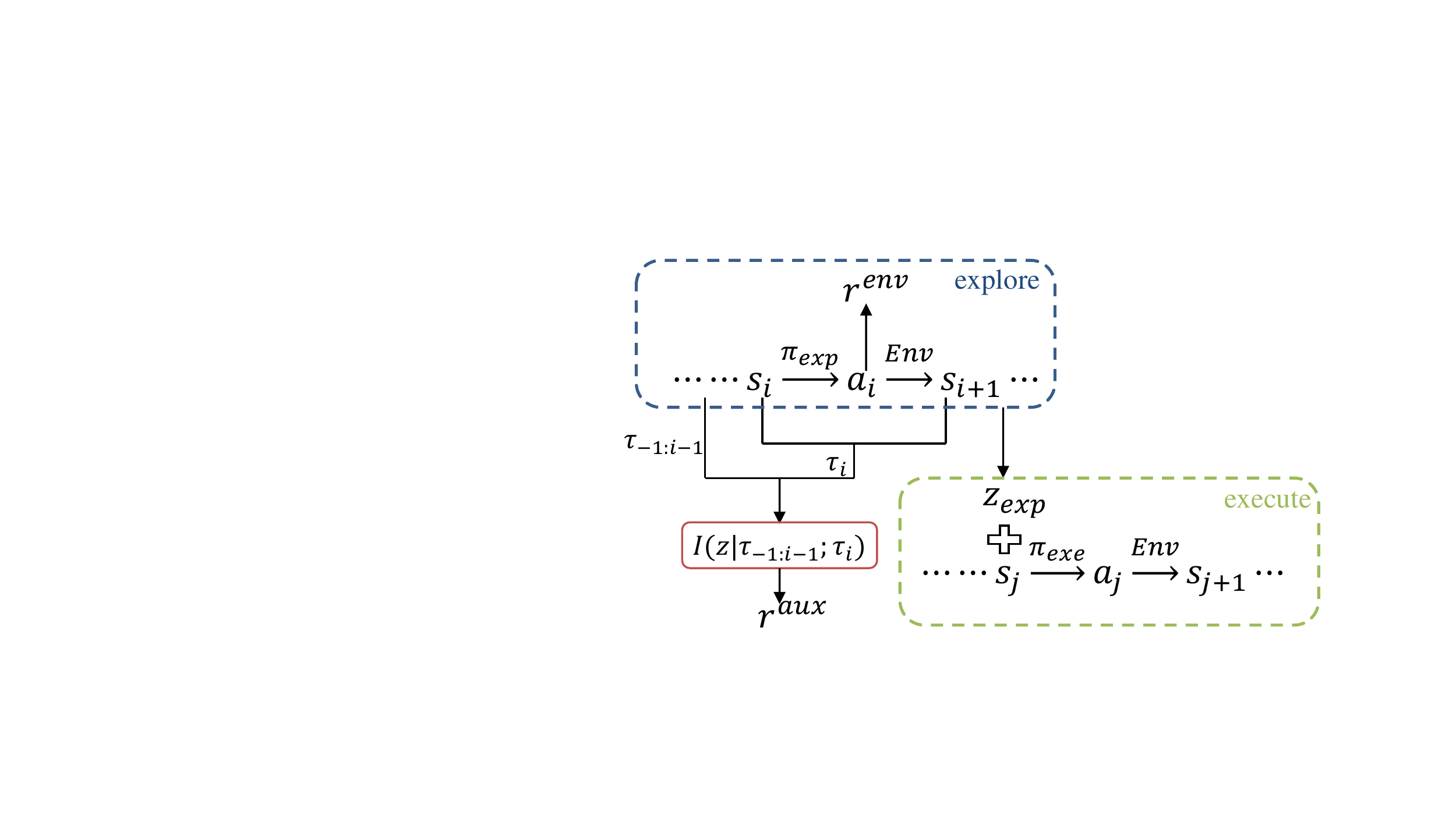}
		\caption{CCM adaptation procedure.}
		\label{figapp1}
	\end{figure}
	
	We model the context encoder as the product of independent Gaussian factors, in which the mean and variance are parameterized by MLPs with $3$ hidden layers of $(300, 300, 300)$ units that produce a 7-dimensional vector. When comparing context encoder training strategy, we use deterministic version of the context encoder network. In other cases, we select $\beta$ in KL divergence term from $\{0.01, 0.1, 1\}$. For contrastive context encoder, the scale of RV or DP loss is $1$ while the scale for contrastive loss is chosen between $\{1, 5, 10\}$. For DP, the penalty parameter~\citep{DBLP:journals/corr/abs-2005-06800} is set to be $0.5$. We use SAC for both exploration and execution agents and set learning rate as $3e-4$. Other hyperparameters are detailed on Table~\ref{tab1} and \ref{tab2}.
	\begin{table}[htb]
		\centering
		\begin{tabular}{lllll}
			\centering
			
			Environment & Meta-train tasks & Meta-test tasks & Number of exploration episodes & $\alpha$  \\\hline 
			cheetah-sparse & 60 & 10 & 2 & 1 \\
			walker-sparse & 60 & 10 & 2 & 2 \\
			hard-point-robot & 40 & 10 & 4 & 2 \\\hline
			
		\end{tabular}
		\caption{CCM's hyperparameters for sparse-reward environments}
		\label{tab1}
	\end{table}
	
	\begin{table}[htb]
		
		\centering
		\begin{tabular}{llll}
			\centering
			
			Environment & Meta-train tasks & Meta-test tasks & Meta batch size \\\hline 
			humanoid-dir & 100 & 30 & 20\\
			cheetah-mass & 30 & 5 & 16\\
			cheetah-mass-OOD & 30 & 5 & 16\\
			cheetah-vel-OOD & 50 & 5 & 24\\
			ant-mass & 50 & 5 & 24\\\hline
			
		\end{tabular}
		\caption{hyperparameters for continuous control benchmarks}
		\label{tab2}
	\end{table}

	\section{Additional Ablation Study}
	\subsection{Further evaluation of contrastive context encoder}
	
	In this section, we show the performance of our proposed contrastive context encoder without combining with other training strategy (i.e. RV or DP). \textit{CCM only} denotes that the context encoder is trained only with gradients from contrastive loss. As shown in Figure~\ref{figapp2}, context encoder trained only with contrastive loss is still able to reach a comparable or better performance than existing methods.
	
	\begin{figure}[htb]
		\centering
		\includegraphics[width=0.9\textwidth]{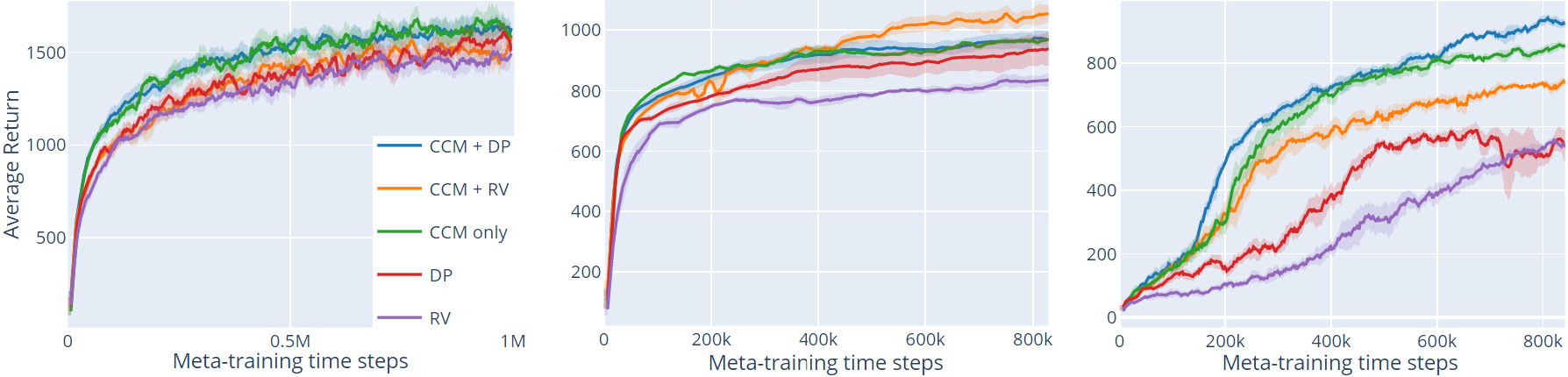}
		\caption{Contrastive context encoder without combining other training strategies}
		\label{figapp2}
	\end{figure}
	\subsection{Intrinsic reward scale}
	We investigate the influence of extra intrinsic reward for CCM's exploration scale by changing the intrinsic reward scale $\alpha$. We experiment with $\alpha \in \{0.1, 1, 10\}$ and show the results on cheetah-sparse. As shown in Figure~\ref{figapp3}, both too large and too small value of intrinsic reward scale will decrease the adaptation performance. Completely depending on intrinsic reward or environment reward will hinder the exploration process.

	\subsection{Context updating frequency}
	We also conduct experiments to show the influence of different context updating frequency during adaptation. We compare the adaptation performance between updating context every episode and updating context every step. As shown in Figure~\ref{figapp3}, updating context every step obtains relatively better returns which is consistent to our original objective in Equation~(\ref{firstmi}).  
	\begin{figure}[htb]
		\centering
		\includegraphics[width=0.65\textwidth]{meta26.png}
		\caption{(a) Context updating frequency comparison;  (b) Intrinsic reward scale comparison}
		\label{figapp3}
	\end{figure}

	\subsection{Influence of the Regularization Term}
	We perform ablation experiments in order to investigate the influence of the regularization term in the objective function which reflects the change in similarity score with all the tasks in embedding space (the second term in (\ref{lowerbound6})). We compare against CCM using intrinsic reward without the regularization term on in-distribution version of hard-point-robot. The results are shown in Figure~\ref{fig4}. The regularization term improves CCM's performance and we assume that without this term the exploration agent may try to collect transitions that the similarity score with all the tasks in embedding space increases as well, which will result in a noisy updating signal.
	
	\begin{figure}[htb]
		\centering
		
		\includegraphics[width=0.35\textwidth]{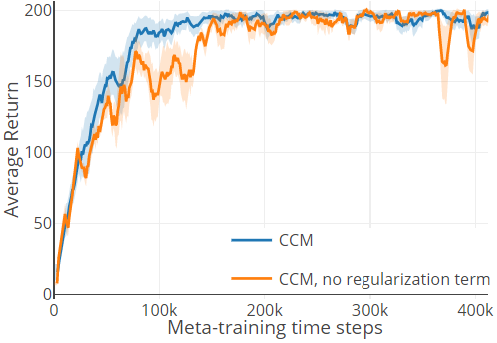}
		
		\caption{Influence of the regularization term}
		
		\label{fig4}
	\end{figure}
	
	\section{Additional Details}
	\subsection{Derivation of Proposition 1.}
	\begin{proof}
		As defined in \cite{DBLP:journals/corr/abs-1807-03748}, the optimal value for $exp(f(c_{1:i-1}, c_{pos}))$ is given by $\frac{p(c_{pos}|c_{1:i-1})}{p(c_{pos})}$, inserting this back into Equation (\ref{lowerbound3}) results in:
		\begin{equation}
		\begin{aligned}
		L_{upper}  &= -\mathop{\mathbb{E}}\limits_{C_{pos}}\log\Big[\frac{\frac{p(c_{pos}|c_{1:i-1})}{p(c_{pos})}}{\sum _{c_{j}\in C_{neg}}\frac{p(c_{j}|c_{1:i-1})}{p(c_{j})}} \Big]\\ &= \mathop{\mathbb{E}}\limits_{C_{pos}}\log\Big[ \frac{p(c_{pos})}{p(c_{pos}|c_{1:i-1})} \sum _{c_{j}\in C_{neg}}\frac{p(c_{j}|c_{1:i-1})}{p(c_{j})} \Big]\\ &\approx \mathop{\mathbb{E}}\limits_{C_{pos}}\log\Big[ \frac{p(c_{pos})}{p(c_{pos}|c_{1:i-1})}(W-1) \mathop{\mathbb{E}}\limits_{c_{neg}} \frac{p(c_{j}|c_{1:i-1})}{p(c_{j})} \Big]\\ &= \mathop{\mathbb{E}}\limits_{C_{pos}}\log\Big[ \frac{p(c_{pos})}{p(c_{pos}|c_{1:i-1})}(W-1)\Big]\\ &\leqslant \mathop{\mathbb{E}}\limits_{C_{pos}}\log\Big[ \frac{p(c_{pos})}{p(c_{pos}|c_{1:i-1})}W \Big]\\ &= -I(c_{pos}; c_{1:i-1}) + \log(W)
		\label{lowerbound4}
		\end{aligned}
		\end{equation}
	\end{proof}
	\subsection{Derivation of Equation (14)}
	
	\begin{equation}
	\begin{aligned}
	L_{upper} - L_{lower} &= \mathop{\mathbb{E}}\limits_{C_{pos}}\log\Big[\frac{\exp(f(c_{1:i}, c_{pos}))}{\sum _{c_{j}\in C}\exp(f(c_{1:i}, c_{j}))} \Big]-\mathop{\mathbb{E}}\limits_{C_{pos}}\log\Big[\frac{\exp(f(c_{1:i-1}, c_{pos}))}{\sum _{c_{j}\in C_{neg}}\exp(f(c_{1:i-1}, c_{j}))} \Big] \\ &= \mathop{\mathbb{E}}\limits_{C_{pos}}[f(c_{1:i}, c_{pos})] - \mathop{\mathbb{E}}\limits_{C_{pos}}\log[\sum _{c_{j}\in C}\exp(f(c_{1:i}, c_{j}))]\\ &~~~-\mathop{\mathbb{E}}\limits_{C_{pos}}[f(c_{1:i-1}, c_{pos})] + \mathop{\mathbb{E}}\limits_{C_{pos}}\log[\sum _{c_{j}\in C_{neg}}\exp(f(c_{1:i-1}, c_{j}))] \\ &= \mathop{\mathbb{E}}\limits_{C_{pos}}\Big[f(c_{1:i},c_{pos}) - f(c_{1:i-1},c_{pos}) \Big] - \mathop{\mathbb{E}}\limits_{C_{pos}}\log\Big[\frac{\sum _{c_{j}\in C}\exp(f(c_{1:i}, c_{j}))}{\sum _{c_{j}\in C_{neg}}\exp(f(c_{1:i-1}, c_{j}))}\Big]
	\end{aligned}
	\end{equation}
	
	\subsection{Connection to recent work on Meta-RL exploration}
	Concurrent with this work, Liu et al.~(\citeyear{DBLP:journals/corr/abs-2008-02790})  and Zhang et al.~(\citeyear{DBLP:journals/corr/abs-2006-08170}) also focus on the exploration problem in Meta-RL and propose to learn execution and exploration with decoupled objectives. Our paper first points out that obtaining a high-quality latent context not only involves exploration to gain informative trajectories but also involves how to train the context encoder to extract the task-specific information. In contrast, those two concurrent work only focuses on the exploration problem. Here we address the latter problem by proposing contrastive context encoder. As for exploration, we theoretically derive a low-variance variational lower bound for the information gain objective based on contrastive loss, without introducing other settings like prediction model. Secondly, the intrinsic reward used by the two concurrent work can be interpreted as the Meta-RL version of "curiosity"~\citep{DBLP:conf/icml/PathakAED17}, which encourages the agent to visit places that is prone to generate high prediction error. However, high prediction error does not always represent informative transition. This might be caused by random noise which is task-irrelevant. In contrast, our objective based on the temporal difference of contrastive loss not only encourages temporal uncertainty of task belief but also limit the agent to explore transitions which is easy to make correct task hypothesis and filter out task-irrelevant information.   
	\clearpage
	\normalsize
\end{document}